\documentclass[preprint,12pt]{article}

\usepackage[top=1.1in, bottom=1.2in, left=0.8in, right=0.8in]{geometry}
\usepackage{amsmath,amssymb,amsthm,mathtools,bm}
\usepackage{booktabs}
\usepackage{graphicx}
\usepackage{microtype}
\usepackage{xcolor}
\usepackage{float}
\usepackage{hyperref}
\hypersetup{
  colorlinks=false,
  pdfborder={0 0 1},
  linkbordercolor={1 0 0},
  citebordercolor={0 1 0},
  urlbordercolor={0 1 1}
}

\newtheorem{theorem}{Theorem}[section]
\newtheorem{proposition}[theorem]{Proposition}
\newtheorem{lemma}[theorem]{Lemma}
\newtheorem{corollary}[theorem]{Corollary}

\newcommand{\R}{\mathbb R}
\newcommand{\E}{\mathbb E}
\newcommand{\Pp}{\mathbb P}
\newcommand{\Law}{\mathcal L}
\newcommand{\KL}{\mathrm{KL}}
\newcommand{\TV}{\mathrm{TV}}
\newcommand{\eps}{\varepsilon}

\title{Score Accuracy Along the Forward Diffusion Does Not Certify Numerical Stability in Diffusion Sampling}

\author{
Yiwei Zhou\thanks{Email: yiwei.zhou@utexas.edu}
}

\date{}

\begin{document}
\maketitle

\begin{abstract}
Score matching controls average error under the forward marginals, but a discretized reverse-time sampler evaluates the learned score along its own trajectory.  We show that small forward-marginal error does not guarantee numerical stability.  We construct a single smooth score field with arbitrarily small forward-marginal $L^2$ error.  The learned reverse-time process is nonexplosive, has moments of every order, and can be arbitrarily close to the exact reverse-time process in path-space total variation.  Yet its Euler--Maruyama discretizations converge in probability while every positive moment diverges.  Thus weak convergence can hold even though every Wasserstein distance $W_p$, $p\ge1$, diverges.

The same failure can occur within one fixed finite neural architecture.  We construct a family of bounded, globally Lipschitz denoisers for which both the forward-marginal error and the path-space total variation distance tend to zero, while their Euler--Maruyama endpoints diverge in every $W_p$.  For compactly supported data, we also give a simple positive result.  Projecting the learned denoiser onto a known bounded closed convex set containing the support preserves pointwise accuracy, gives grid-uniform moment bounds, and yields Wasserstein convergence under mild local regularity.  Experiments with a small fixed DiT-style network show large growth along rare numerical trajectories and its suppression by denoiser projection, while overall trajectory errors remain small.


\end{abstract}

\section{Introduction}

In score-based diffusion models, the exact reverse-time process uses the true score \(s_t\) in its drift.  Replacing \(s_t\) by a learned approximation \(\widehat s_t\) gives the learned reverse-time process, which is then  Euler--Maruyama discretized for sampling.

A standard score-learning error has the form
\begin{equation}
\mathcal E_{\mathrm{score}}
=
\int_\delta^T
\E_{X_t\sim p_t}
\|\widehat s_t(X_t)-s_t(X_t)\|^2\,dt.
\label{eq:score-budget}
\end{equation}
We call it \emph{on-path} score error because it is measured under the forward marginals.  A natural question is:

``Suppose both the exact and learned reverse-time processes are stable, and the learned process can be made arbitrarily close to the exact one in path-space total variation. Is an arbitrarily small on-path score error then enough to guarantee stability of the discretization?"

Perhaps surprisingly, the answer is no.  We construct a single smooth score field with arbitrarily small on-path error.  The learned reverse-time process is nonexplosive, has moments of every order, and can be arbitrarily close to the exact reverse-time process in path-space total variation.  Yet its Euler--Maruyama endpoints converge in probability while every positive moment diverges.  Weak convergence holds while the Wasserstein distance $W_p$ diverges for every $p\ge1$.  The same mechanism extends to arbitrary dimension and to probability-flow ODEs.

Furthermore, this error alone does not rank samplers by numerical stability.  Two smooth score fields can have the same arbitrarily small on-path error while one Euler--Maruyama discretization has grid-uniform moments of every order and the other loses every positive moment.  The same obstruction persists even within a fixed finite GELU architecture: bounded, globally Lipschitz denoisers can have arbitrarily small on-path error and path-space total variation distance while their Euler--Maruyama endpoints diverge in every $W_p$.

The mechanism is simple.  We place an inward superlinear perturbation in a remote region carrying very little mass under the forward marginals.  The learned reverse-time process rarely reaches this region, so its path law remains close to that of the exact reverse-time process.  A rare numerical trajectory can nevertheless enter it, after which explicit updates trigger repeated superlinear amplification.  The growth along these rare trajectories outweighs their probability cost and destroys every positive moment.  This mechanism is related to the classical instability of explicit Euler schemes for superlinear SDEs \cite{hjk2011}, but here the unstable region carries negligible weight under the on-path score error while the learned reverse-time process remains stable.

We also give a simple positive result.  Suppose the data are supported in a known bounded closed convex set \(C\).  Since the exact denoiser \(D_t(x)=\E[X_0\mid X_t=x]\) lies in \(C\), we project the learned denoiser onto \(C\) before forming the score.  This projection cannot increase the pointwise denoiser or score error, while making the deterministic part of each Euler--Maruyama update pull toward a bounded set.  The projected chain has grid-uniform moment bounds, so weak numerical convergence implies Wasserstein convergence. For image data scaled to a box, this projection coincides with clipping the predicted clean sample, a heuristic used in image diffusion models \cite{saharia2022imagen,lou2023reflected}.  Our result provides a stability interpretation of this practice.

Related work studies diffusion sampling under distribution-weighted score error and additional assumptions that yield path-space or endpoint guarantees \cite{chen2022minimal,lee2022general,benton2024linear}.  Wasserstein convergence has also been established under assumptions on the data law, score regularity, or dynamics \cite{gao2023wasserstein,brunosabanis2025}, while contraction has been used as a direct stability principle for reverse-time sampling \cite{tangzhao2024contractive}.  Recent work has questioned what scalar score-error measures can certify \cite{cao2026robustness,khelifa2026gradients}.  Our results identify a different obstruction: even when the learned continuous-time dynamics are stable and close to the exact process, explicit discretization can lose all moment and Wasserstein stability because numerical trajectories may move into regions poorly represented by the forward marginals.

\section{Setup}

We work with the variance-exploding (VE) Gaussian noising process
\begin{equation*}
X_t=X_0+\sqrt{t}\,g,
\qquad g\sim N(0,I_d),
\qquad t\in[0,T].
\end{equation*}
Let $p_t$ denote the law of $X_t$.  Tweedie's identity gives
\begin{equation*}
s_t(x)
=\nabla\log p_t(x)
=\frac{D_t(x)-x}{t},
\qquad
D_t(x)=\E[X_0\mid X_t=x].
\end{equation*}
For a learned score $\widehat s_t$, we use the associated denoiser parameterization
\[
\widehat D_t(x)=x+t\widehat s_t(x).
\]
When $X_0$ is supported in a compact set $\mathcal K\subset B(0,M)$, let
\[
C_0=\overline{\operatorname{conv}}(\mathcal K).
\]
Conditional expectation then implies
\begin{equation*}
D_t(x)\in C_0,
\qquad
\|s_t(x)\|\le \frac{M+\|x\|}{t}.
\end{equation*}
Thus, for bounded data and $t\ge\delta>0$, the exact score grows at most linearly in space.  Our concern is different: an approximate score $\widehat s_t$ may be poorly controlled in regions that carry little mass under the forward marginals.

We consider a decreasing grid
\[
T=t_K>t_{K-1}>\cdots>t_0=\delta>0,
\qquad
h_k=t_{k+1}-t_k,
\qquad
\rho_k=\frac{h_k}{t_{k+1}}.
\]
Because the grid is strictly decreasing and positive, every step satisfies $0<\rho_k<1$.
A simple reverse Euler step driven by a score field $s$ has the form
\begin{equation}
Y_k=Y_{k+1}+h_k s_{t_{k+1}}(Y_{k+1})+\sqrt{h_k}\,\xi_k,
\label{eq:reverse-em}
\end{equation}
with independent standard Gaussian $\xi_k$.  The chain runs from $k=K$ ($t=T$) down to $k=0$ ($t=\delta$), and $Y_0$ is the generated sample.

\paragraph{Notation.}  $X$ denotes the forward noising process, $U$ the exact reverse-time process, $Z$ the learned reverse-time process, and $Y$ an Euler--Maruyama chain.  The exact, learned, and projected scores are $s$, $\widehat s$, and $s^P$; the corresponding denoisers are $D$, $\widehat D$, and $D^P$.  For a random variable or process $V$, $\Law(V)$ denotes its law.  We write $\TV(\nu,\mu)$ for total variation distance, $\KL(\nu\|\mu)$ for relative entropy, and $\mathcal P_p(\R^d)$ for the probability measures on $\R^d$ with finite $p$th moment.

\section{The blind spot of on-path score accuracy}

Since $p_t$ has rapidly decaying tails, the score can be changed far outside the typical region while adding only an arbitrarily small amount to the $L^2(p_t)$ budget.  The main difficulty is to turn this weak visibility into numerical instability.  A one-step perturbation of amplitude $a_R$ at radius $R$ pays essentially the same rare-event factor in both the score-error budget and the resulting moment error: under a Gaussian-tailed forward marginal, the budget contribution is of order
\[
a_R^2 e^{-cR^2}.
\]
The separation therefore needs repeated evaluation.  Once an Euler trajectory enters the perturbed tail, successive explicit updates can compound superlinear growth faster than the entry probability decays.

\subsection{A fixed score field with weak convergence and Wasserstein divergence}

We fix the learned score and time horizon and vary only the grid resolution.

\begin{theorem}[Weak convergence with Wasserstein divergence]
\label{thm:strong-noncert}
Let $\mu_0$ be any probability law supported in $[-M,M]$, let
$p_t=\mu_0*N(0,t)$, and fix $0<\delta<T<\infty$.  For every
$\eps>0$, $\eta>0$, $\lambda>0$, and $\alpha>0$, there exist $R>0$ and an even cutoff
$\chi_R\in C^\infty(\R;[0,1])$, with $\chi_R=0$ on $[-R,R]$ and
$\chi_R=1$ outside $[-R-1,R+1]$, such that the single score field
\begin{equation}
\widehat s_t(x)=s_t(x)-\lambda\chi_R(x)x|x|^\alpha,
\qquad t\in[\delta,T],
\label{eq:inward-superlinear-corruption}
\end{equation}
has all of the following properties.
\begin{enumerate}
\item \emph{On-path accuracy.} Its error under the forward marginals is uniformly small in time:
\begin{equation}
\sup_{t\in[\delta,T]}
\E_{p_t}|\widehat s_t-s_t|^2\le\eps.
\label{eq:small-score-error-main}
\end{equation}
\item \emph{Stable and nearly exact continuous dynamics.} There exist constants $c,C>0$ such that, uniformly in
$t\in[\delta,T]$,
\begin{equation}
x\widehat s_t(x)\le C-c|x|^{2+\alpha},
\qquad x\in\R.
\label{eq:global-coercive-score}
\end{equation}
Consequently the learned reverse-time process
\begin{equation}
dZ_u
=
\widehat s_{T-u}(Z_u)\,du+dB_u,
\qquad
Z_0\sim p_T,
\qquad
0\le u\le T-\delta,
\label{eq:learned-continuous-reverse}
\end{equation}
has a unique nonexplosive strong solution and, for every $q\ge2$,
\begin{equation}
\sup_{0\le u\le T-\delta}\E|Z_u|^q<\infty.
\label{eq:continuous-moment-stability}
\end{equation}
Let $U$ denote the exact reverse-time process
\begin{equation*}
dU_u=s_{T-u}(U_u)\,du+dB_u,
\qquad U_0\sim p_T,
\end{equation*}
driven by the same Brownian normalization and terminal law.  The two continuous path laws can be made arbitrarily close:
\begin{equation}
\TV\!\left(\Law\bigl((Z_u)_{0\le u\le T-\delta}\bigr),
\Law\bigl((U_u)_{0\le u\le T-\delta}\bigr)
\right)\le\eta.
\label{eq:path-law-tv-close}
\end{equation}
\item \emph{Discrete moment failure despite weak convergence.} Let $K_m=2^m$ and use the nested geometric grids
\begin{equation}
t_j^{(K_m)}
=
\delta\exp\!\left(\frac{j}{K_m}\log\frac{T}{\delta}\right),
\qquad
j=0,\ldots,K_m.
\label{eq:nested-geometric-grid}
\end{equation}
Initialize $Y_{K_m}^{(K_m)}\sim p_T$ and apply the Euler--Maruyama recursion
\eqref{eq:reverse-em} with the same field \eqref{eq:inward-superlinear-corruption} on every grid.  Couple the Euler noises to the Brownian increments of \eqref{eq:learned-continuous-reverse}.  Then
\begin{equation}
Y_0^{(K_m)}\longrightarrow Z_{T-\delta}
\qquad\text{in probability}.
\label{eq:euler-in-probability}
\end{equation}
Nevertheless, for every $q>0$,
\begin{equation}
\E|Y_0^{(K_m)}|^q\longrightarrow\infty
\qquad\text{as }m\to\infty.
\label{eq:all-moment-divergence}
\end{equation}
Consequently the Euler laws converge weakly to the learned endpoint, but for every $p\ge1$ and every fixed $\nu\in\mathcal P_p(\R)$,
\begin{equation}
W_p\bigl(\Law(Y_0^{(K_m)}),\nu\bigr)\longrightarrow\infty.
\label{eq:wp-any-target-divergence}
\end{equation}
In particular this holds for $\nu=\Law(Z_{T-\delta})$, $\nu=p_\delta$, and $\nu=\mu_0$.
\end{enumerate}
\end{theorem}

\begin{proof}
\medskip\noindent\textbf{Step 1: make the on-path error small.}\quad
The exact score satisfies the global linear envelope
\begin{equation}
|s_t(x)|
=\frac{|D_t(x)-x|}{t}
\le\frac{|x|+M}{\delta},
\qquad
(x,t)\in\R\times[\delta,T].
\label{eq:exact-linear-envelope-proof}
\end{equation}
Under the coupling $X_t=X_0+\sqrt t g$ with $|X_0|\le M$ almost surely,
\begin{align*}
\sup_{t\in[\delta,T]}
\E_{p_t}|\widehat s_t-s_t|^2
&\le
\lambda^2
\sup_{t\in[\delta,T]}
\E\bigl[|X_t|^{2+2\alpha}\mathbf 1_{\{|X_t|>R\}}\bigr]\nonumber\\
&\le
\lambda^2
\E\Bigl[(M+\sqrt T|g|)^{2+2\alpha}
\mathbf 1_{\{M+\sqrt T|g|>R\}}\Bigr].
\end{align*}
The last expression tends to zero as $R\to\infty$ by dominated convergence.  We can therefore choose $R$ so that \eqref{eq:small-score-error-main} holds.

\medskip\noindent\textbf{Step 2: control the continuous-time process.}\quad
For $|x|\ge R+1$, the cutoff is fully active, and
\eqref{eq:exact-linear-envelope-proof} gives
\[
x\widehat s_t(x)
\le
\frac{|x|^2+M|x|}{\delta}-\lambda|x|^{2+\alpha}.
\]
As $|x|\to\infty$, the negative superquadratic term dominates the quadratic and linear terms.  Hence there is $R_0<\infty$ such that, uniformly in $t\in[\delta,T]$,
\[
x\widehat s_t(x)
\le
-\frac{\lambda}{2}|x|^{2+\alpha},
\qquad |x|\ge R_0.
\]
On the compact set $[\delta,T]\times[-R_0,R_0]$, the continuous function
\[
(t,x)\longmapsto
x\widehat s_t(x)+\frac{\lambda}{4}|x|^{2+\alpha}
\]
is bounded above.  Enlarging the constant if necessary therefore gives
\eqref{eq:global-coercive-score} for suitable $c,C>0$.

We next check local well-posedness.  For every $t\ge\delta$, $p_t$ is a positive Gaussian convolution, so the exact score $s_t=\partial_x\log p_t$ is smooth on $[\delta,T]\times\R$.  The cutoff $\chi_R$ vanishes near the origin, so $\chi_R(x)x|x|^\alpha$ is smooth even when $\alpha$ is nonintegral.  Thus $\widehat s_t$ is smooth.  The reverse drift $b(u,x)=\widehat s_{T-u}(x)$ is therefore locally Lipschitz in $x$, uniformly in $u\in[0,T-\delta]$: for every $H<\infty$,
\[
\sup_{\substack{0\le u\le T-\delta\\ |x|\le H}}
|\partial_x b(u,x)|<\infty.
\]
Hence the learned reverse-time process admits a unique strong solution up to its
explosion time.

For the time-inhomogeneous generator
\[
\mathcal L_u f(x)
=
b(u,x)f'(x)+\frac12 f''(x),
\]
consider $V_m(x)=(1+x^2)^m$.  Using \eqref{eq:global-coercive-score}, the drift term satisfies
\[
2m(1+x^2)^{m-1}x\,b(u,x)
\le
2m(1+x^2)^{m-1}(C-c|x|^{2+\alpha}).
\]
The diffusion term $\frac12V_m''(x)$ has lower polynomial order, so for every integer $m\ge1$ there is $C_m<\infty$, uniform in $u\in[0,T-\delta]$, such that
\[
\mathcal L_u V_m(x)\le C_m V_m(x).
\]
For $m=1$, apply It\^o's formula up to the exit time from $[-n,n]$ and then let $n\to\infty$.  The resulting bound rules out explosion.  The same stopped argument for general $m$, followed by Gr\"onwall's inequality, gives uniform bounds for all even moments on the finite horizon.  Since $p_T$ has moments of every order, interpolation gives \eqref{eq:continuous-moment-stability} for every real $q\ge2$.

\medskip\noindent\textbf{Step 3: compare the learned and exact path laws.}\quad
Next, couple the learned and exact reverse-time processes with the same initial draw from $p_T$ and the same Brownian motion.  Because $\widehat s_t=s_t$ on $[-R,R]$, pathwise uniqueness implies that the two solutions agree until their first exit from this interval.  The coupling inequality therefore gives
\begin{equation*}
\TV\!\left(\Law((Z_u)_{u\le T-\delta}),
\Law((U_u)_{u\le T-\delta})
\right)
\le
\Pp\!\left(\sup_{0\le u\le T-\delta}|U_u|\ge R\right).
\end{equation*}
The exact reverse-time process has continuous paths on a compact time interval, so the probability on the right tends to zero as $R\to\infty$.  We may therefore enlarge the radius chosen in Step~1 until this probability is at most $\eta$.  Enlarging $R$ can only decrease the on-path error bound, while the coercivity argument of Step~2 still applies with new constants.  This gives \eqref{eq:path-law-tv-close}.

\medskip\noindent\textbf{Step 4: prove Euler convergence in probability.}\quad
Write $b(u,x)=\widehat s_{T-u}(x)$.  Let
$\pi_H(x)=(-H)\vee x\wedge H$ and define the truncated drift
$b_H(u,x)=b(u,\pi_H(x))$.  Since $b$ is smooth on
$[0,T-\delta]\times[-H,H]$, the truncated drift is globally Lipschitz in
space, uniformly in time, and uniformly continuous in time on the compact
horizon.  Let $Z^H$ and $Y^{H,K}$ be the corresponding
SDE and Euler approximation, coupled with the same initial condition and
Brownian increments.  Standard Euler convergence for globally Lipschitz,
time-inhomogeneous drifts on meshes with vanishing maximal step gives \cite{highammaostuart2002}
\begin{equation}
\max_{0\le j\le K}
\left|Y_{K-j}^{H,K}-Z^H_{u_j}\right|
\longrightarrow0
\quad\text{in probability},
\label{eq:truncated-euler-convergence}
\end{equation}
where $u_j=T-t_{K-j}$ is the corresponding increasing-time grid.

Fix $\zeta>0$.  By nonexplosion and continuity of $Z$, choose $H$ so large that
\[
\Pp\!\left(\sup_{u\le T-\delta}|Z_u|\ge H-1\right)<\zeta.
\]
On the complementary event, $Z^H=Z$.  By \eqref{eq:truncated-euler-convergence}, the truncated Euler path stays within distance $1$ of $Z^H$ with probability tending to at least $1-\zeta$.  It then remains inside $[-H,H]$, so the raw and truncated Euler recursions agree step by step.  Thus, for every $r>0$,
\[
\limsup_{K\to\infty}
\Pp\bigl(|Y_0^{(K)}-Z_{T-\delta}|>r\bigr)
\le \zeta.
\]
Since $\zeta$ is arbitrary, this proves \eqref{eq:euler-in-probability}.

\medskip\noindent\textbf{Step 5: amplify a rare terminal shell.}\quad
It remains to prove divergence of the explicit moments.  Put
$\ell=\log(T/\delta)$.  For the grid \eqref{eq:nested-geometric-grid},
\[
h_j=t_{j+1}-t_j=t_j(e^{\ell/K}-1).
\]
For all sufficiently large $K$ there are constants $a,b>0$, depending only on $(\delta,T)$, such that
\begin{equation}
\frac aK\le h_j\le\frac bK,
\qquad j=0,\ldots,K-1.
\label{eq:geometric-step-comparison}
\end{equation}
Consider one reverse Euler step from $Y_{j+1}=y$ with $|y|\ge R+1$.
On the event $|\xi_j|\le1$, \eqref{eq:exact-linear-envelope-proof} and
\eqref{eq:geometric-step-comparison} give
\begin{align*}
|Y_j|
&\ge
\lambda h_j|y|^{1+\alpha}-|y|-h_j|s_{t_{j+1}}(y)|-\sqrt{h_j}|\xi_j|\\
&\ge
\frac{\lambda a}{K}|y|^{1+\alpha}
-\left(1+\frac{b}{\delta K}\right)|y|
-\frac{bM}{\delta K}
-\sqrt{\frac bK}.
\end{align*}
Choose $A>0$, depending only on $\lambda,\alpha,a,b,\delta,M$, so large that for every sufficiently large $K$ and every $|y|\ge A K^{1/\alpha}$,
\begin{equation*}
|Y_j|\ge\frac{\lambda a}{4K}|y|^{1+\alpha}.
\end{equation*}
We also take $K$ large enough that $A K^{1/\alpha}\ge R+1$.
Let $c_0=\lambda a/4$ and $z_j=|Y_j|/K^{1/\alpha}$.  Then
\[
z_j\ge c_0 z_{j+1}^{1+\alpha}.
\]
Taking $A$ still larger if necessary so that
$c_0^{1/\alpha}A>1$ and $c_0A^{1+\alpha}\ge A$, induction gives, after $n$ reverse steps,
\begin{equation}
z_{K-n}
\ge
c_0^{-1/\alpha}
\left(c_0^{1/\alpha}A\right)^{(1+\alpha)^n}.
\label{eq:double-exponential-growth}
\end{equation}
In particular every iterate remains in the active tail on the event used below.

Define
\[
E_K
=
\{Y_K\in[A K^{1/\alpha},A K^{1/\alpha}+1]\}
\cap
\bigcap_{j=0}^{K-1}\{|\xi_j|\le1\}.
\]
For $y\ge0$,
\[
p_T(y)
=
\E\left[\frac1{\sqrt{2\pi T}}
\exp\!\left(-\frac{(y-X_0)^2}{2T}\right)\right]
\ge
\frac1{\sqrt{2\pi T}}
\exp\!\left(-\frac{(y+M)^2}{2T}\right).
\]
Therefore there are $C_1,C_2>0$, independent of $K$, such that
\begin{equation*}
\Pp\bigl(Y_K\in[A K^{1/\alpha},A K^{1/\alpha}+1]\bigr)
\ge C_1\exp(-C_2K^{2/\alpha}).
\end{equation*}
Since $p_\xi:=\Pp(|\xi_0|\le1)>0$ and the Euler noises are independent of $Y_K$ and of one another,
\begin{equation}
\Pp(E_K)
\ge
C_1\exp(-C_2K^{2/\alpha})p_\xi^K
\ge
\exp\!\left[-C_3\bigl(K^{2/\alpha}+K\bigr)\right]
\label{eq:full-rare-event-probability}
\end{equation}
for a constant $C_3<\infty$.  Combining
\eqref{eq:double-exponential-growth} at $n=K$ with
\eqref{eq:full-rare-event-probability}, for every $q>0$,
\begin{align*}
\E|Y_0^{(K)}|^q
&\ge
\Pp(E_K)
K^{q/\alpha}c_0^{-q/\alpha}
\left(c_0^{1/\alpha}A\right)^{q(1+\alpha)^K}
\longrightarrow\infty.
\end{align*}
Indeed, the logarithm of the amplification term is of order $(1+\alpha)^K$, whereas the logarithmic rare-event cost is only of order $K^{2/\alpha}+K$.
This proves \eqref{eq:all-moment-divergence}.  For each fixed $K$, finite-step induction using the growth bound
$|\widehat s_t(x)|\le C(1+|x|^{1+\alpha})$ and Gaussian moments gives
$\E|Y_0^{(K)}|^q<\infty$ for every $q<\infty$.  Equation \eqref{eq:euler-in-probability} implies weak convergence of the endpoint laws.

Finally fix $p\ge1$, $\nu\in\mathcal P_p(\R)$, and any coupling $(V,V')$ with
$V\sim\Law(Y_0^{(K)})$ and $V'\sim\nu$.  Minkowski's inequality gives
\[
\bigl(\E|V-V'|^p\bigr)^{1/p}
\ge
\bigl(\E|V|^p\bigr)^{1/p}
-
\bigl(\E|V'|^p\bigr)^{1/p}.
\]
Taking the infimum over couplings and using the $q=p$ case of
\eqref{eq:all-moment-divergence} proves \eqref{eq:wp-any-target-divergence}.
\end{proof}

By a standard time-dependent rescaling, the same counterexample can be expressed in the usual VP/DDPM coordinates.  We use the present VE normalization throughout for simplicity.

The amplification mechanism is related to the classical divergence of explicit Euler schemes for superlinear SDEs \cite{hjk2011}.  The difference is that here an arbitrarily mild superlinear perturbation can be placed in a remote region with arbitrarily small forward-marginal mass.  It is dissipative in continuous time but can trigger a discrete cascade through explicit overshoot.

Strong uniform global control also rules out the fixed-field construction.  Examples include spatial Lipschitz and linear-growth bounds that hold uniformly in time and across the approximating family.  Such assumptions give standard grid-uniform finite-horizon Euler moment bounds, but they are much stronger than on-path score accuracy and can be restrictive or hard to verify in practice \cite{conforti2025kl}.

\subsection{Further consequences of the counterexample}

\paragraph{Higher dimensions.}
The same mechanism extends to any dimension through a radial perturbation.

\begin{corollary}[The obstruction persists in every dimension]
\label{cor:dimension-extension}
Let $d\ge1$ and suppose $\mu_0$ is supported in $B(0,M)\subset\R^d$.  For every $\eps>0$, $\eta>0$, $\lambda>0$, and $\alpha>0$, there is a radial cutoff $\chi_R$ such that
\begin{equation*}
\widehat s_t(x)
=
 s_t(x)-\lambda\chi_R(\|x\|)\|x\|^\alpha x
\end{equation*}
has the $d$-dimensional counterparts of all three conclusions of Theorem~\ref{thm:strong-noncert}.  Its on-path $L^2(p_t)$ error is uniformly small.  The learned reverse-time process is nonexplosive, has moments of every order, and is within $\eta$ of the exact reverse-time process in path-space total variation.  The Euler--Maruyama discretizations converge in probability, while every positive moment diverges.  Their $W_p$ distance from every fixed law in $\mathcal P_p(\R^d)$ also diverges.
\end{corollary}

\begin{proof}
The proof of Theorem~\ref{thm:strong-noncert} carries over radially.  Since $\|D_t(x)\|\le M$,
\[
x\cdot\widehat s_t(x)\le C-c\|x\|^{2+\alpha},
\]
so the Lyapunov, path-coupling, and localization arguments are unchanged with $|\cdot|$ replaced by $\|\cdot\|$.  On the active tail, the radial norm satisfies the same superlinear one-step lower bound up to fixed-dimensional linear and noise terms.  On the shell $A K^{1/\alpha}\le\|Y_K\|\le A K^{1/\alpha}+1$ and the bounded-noise event $\|\xi_j\|\le1$, choosing $A$ large yields the same normalized $(1+\alpha)$-power recursion as in Theorem~\ref{thm:strong-noncert}.  The shell probability is bounded below by $\exp(-C K^{2/\alpha})$ up to a polynomial factor, and the bounded-noise event contributes $\exp(-c_d K)$.  The iterated amplification therefore dominates the rare-event cost.
\end{proof}

\paragraph{Probability flow ODE.}
The same separation appears for deterministic probability flow.  In our VE normalization, the exact reverse-time probability flow ODE is
\[
\dot U_u=\frac12 s_{T-u}(U_u),
\qquad U_0\sim p_T.
\]
The rare-shell and repeated-amplification mechanism also has a deterministic version for this flow.

\begin{corollary}[Deterministic probability flow analogue]
\label{cor:probability-flow}
The same separation can occur for the probability flow ODE\@.  Fix $\eps>0$, $\eta>0$, $\lambda>0$, and $\alpha>0$.  The radius $R$ in \eqref{eq:inward-superlinear-corruption} can be chosen so that the conclusions of Theorem~\ref{thm:strong-noncert} hold and the learned probability flow ODE
\begin{equation*}
\dot Z_u
=
\frac12\widehat s_{T-u}(Z_u),
\qquad
Z_0\sim p_T,
\end{equation*}
is nonexplosive and has moments of every order.  The path-space total variation distance between the learned and exact probability flow path laws is at most $\eta$.  Its descending-grid explicit Euler scheme
\begin{equation*}
Q_j^{(K)}
=
Q_{j+1}^{(K)}
+
\frac{h_j}{2}\widehat s_{t_{j+1}}(Q_{j+1}^{(K)}),
\qquad Q_K^{(K)}\sim p_T,
\end{equation*}
can be coupled so that
\[
Q_0^{(K_m)}\longrightarrow Z_{T-\delta}
\quad\text{in probability},
\qquad
\E\|Q_0^{(K_m)}\|^q\longrightarrow\infty
\quad\text{for every }q>0.
\]
Hence the deterministic sampler laws converge weakly while their $W_p$ distance
from every fixed law in $\mathcal P_p(\R)$ diverges for every $p\ge1$.
\end{corollary}

\begin{proof}
For the VE forward diffusion, the probability flow ODE has reverse-time drift
$\frac12 s_{T-u}$ \cite{song2021sde}.  The coercive estimate and the common-core
coupling argument used for the SDE apply after multiplying the drift by $1/2$.
If necessary, enlarge $R$ once more so that the exact SDE and exact ODE exit probabilities are both at most $\eta$; all earlier score-error and coercivity conclusions are preserved.  Local truncation gives convergence in probability of the explicit Euler discretization.  Choose a sufficiently large shell constant $A>0$.  On the terminal shell $A K^{1/\alpha}\le |Q_K^{(K)}|\le A K^{1/\alpha}+1$, the active-tail update obeys
\[
|Q_j^{(K)}|
\ge c h_j\lambda |Q_{j+1}^{(K)}|^{1+\alpha}.
\]
No per-step Gaussian event is needed, so the rare-event
cost is only the terminal-shell probability, bounded below by $\exp(-C K^{2/\alpha})$.
The same iterated $(1+\alpha)$-power recursion yields super-exponential conditional growth, which proves
all moment and Wasserstein conclusions exactly as in Step 5 of
Theorem~\ref{thm:strong-noncert}.
\end{proof}

\paragraph{Equal-budget non-ordering.}

The main theorem shows that a small on-path score-error budget alone cannot certify stability.  The next result goes further: this scalar budget does not even rank score fields by numerical stability.

\begin{corollary}[Equal-budget non-ordering]
\label{cor:equal-budget}
For every $\varepsilon>0$, there exist $b\in(0,\varepsilon]$ and two smooth
score fields $\widehat s^{\mathrm{safe}}$ and $\widehat s^{\mathrm{bad}}$ such
that, when both Euler--Maruyama chains are initialized from $p_T$,
\begin{equation}
\int_\delta^T
\E_{p_t}\left|
\widehat s_t^{\mathrm{safe}}(X_t)-s_t(X_t)
\right|^2\,dt
=
\int_\delta^T
\E_{p_t}\left|
\widehat s_t^{\mathrm{bad}}(X_t)-s_t(X_t)
\right|^2\,dt
=b.
\label{eq:equal-budget-identity}
\end{equation}
On the nested geometric grids \eqref{eq:nested-geometric-grid}, the Euler--Maruyama discretization
associated with $\widehat s^{\mathrm{safe}}$ has grid-uniform moments of every
order:
\begin{equation}
\sup_m \E\left|Y_0^{(K_m),\mathrm{safe}}\right|^q<\infty,
\qquad q\ge1,
\label{eq:safe-grid-uniform-moments}
\end{equation}
whereas the scheme associated with $\widehat s^{\mathrm{bad}}$ satisfies
\begin{equation}
\E\left|Y_0^{(K_m),\mathrm{bad}}\right|^q\longrightarrow\infty,
\qquad q>0.
\label{eq:bad-equal-budget-divergence}
\end{equation}
\end{corollary}

\begin{proof}
Apply Theorem~\ref{thm:strong-noncert} with tolerance
$\varepsilon/(T-\delta)$ and any fixed positive choices of $\eta$, $\lambda$, and $\alpha$.  The resulting bad field has integrated budget at
most $\varepsilon$; set
\[
b
:=
\int_\delta^T
\E_{p_t}\left|
\widehat s_t^{\mathrm{bad}}(X_t)-s_t(X_t)
\right|^2\,dt.
\]
The Gaussian-smoothed marginals have strictly positive density and the
corruption is nonzero on a set of positive Lebesgue measure, so $b>0$.
Theorem~\ref{thm:strong-noncert} gives
\eqref{eq:bad-equal-budget-divergence}.

For the safe comparator, choose a constant $a\in\R$ with
\[
a^2=\frac{b}{T-\delta},
\qquad
\widehat s_t^{\mathrm{safe}}(x)=s_t(x)+a.
\]
Then \eqref{eq:equal-budget-identity} holds exactly.  Its Euler update is
\[
Y_k
=
\frac{t_k}{t_{k+1}}Y_{k+1}
+
\frac{h_k}{t_{k+1}}D_{t_{k+1}}(Y_{k+1})
+
h_k a
+
\sqrt{h_k}\,\xi_k.
\]
Since $D_t(x)\in[-M,M]$, iterating the affine recursion writes $Y_k$ as four terms: a contracted terminal state, a convex combination of bounded denoiser values, a deterministic shift
\[
a t_k\sum_{j=k}^{K_m-1}\frac{h_j}{t_j},
\]
and a centered Gaussian term.  On the geometric grid,
$t_{j+1}/t_j=\exp(\ell/K_m)$ with $\ell=\log(T/\delta)$, so
\[
\sup_m K_m\bigl(e^{\ell/K_m}-1\bigr)<\infty.
\]
Hence the deterministic shift is uniformly bounded, while the Gaussian term
has uniformly bounded variance by a direct geometric-series estimate.  Since $Y_{K_m}\sim p_T$ has moments of every order,
\eqref{eq:safe-grid-uniform-moments} follows.
\end{proof}

\subsection{A fixed-architecture neural obstruction}

 The preceding counterexample is constructed at the level of a general smooth score field.  We now show that the same obstruction persists within a fixed finite neural architecture.  Globally bounded and globally Lipschitz denoisers represented by the same GELU architecture can have arbitrarily small on-path score error and path-space total variation distance, while their Euler--Maruyama discretizations remain unstable in every Wasserstein metric.  Unlike the fixed-field construction above, this result considers a sequence of network parameters within one fixed architecture.

For the construction, define
\[
\psi(u)=u_+-(u-1)_+,
\qquad
\Gamma_R(x)=\psi(x-R)+\psi(-x-R).
\]
Thus $\Gamma_R=0$ on $[-R,R]$, $0\le \Gamma_R\le1$, and $\Gamma_R=1$ outside $[-R-1,R+1]$.  Appendix~\ref{app:fixed-neural-realization} constructs the fixed adaptive LayerNorm-Zero (adaLN-Zero) DiT and proves the approximation and polynomial output bounds used below.

\begin{theorem}[Fixed-DiT neural non-certification]
\label{thm:fixed-dit-noncert}
Fix $\lambda>0$.  Take $\mu_0=\delta_0$, so that $p_t=N(0,t)$ and the exact denoiser is zero.  There exist one fixed scalar one-token adaLN-Zero DiT architecture, a sequence of parameter values $\theta_m$ in this same architecture, and geometric grids with $K_m\to\infty$ such that, with
\[
\widehat s_t^{(m)}(x)=\frac{D_{\theta_m}(x)-x}{t},
\qquad t\in[\delta,T],
\]
and with both the learned reverse-time processes and Euler--Maruyama chains initialized from $p_T$, the following hold:
\begin{enumerate}
\item every $D_{\theta_m}$ is globally bounded and globally Lipschitz;
\item the on-path score error vanishes,
\[
\sup_{t\in[\delta,T]}
\E_{p_t}|\widehat s_t^{(m)}-s_t|^2\to0;
\]
\item the learned reverse-time processes are nonexplosive, have moments of every order, and their path laws converge in total variation to the exact reverse-time path law;
\item the Euler--Maruyama chains can be coupled with the exact reverse-time process so that their endpoints converge in probability to the exact reverse-time endpoint, while for every $q>0$ and every $p\ge1$,
\[
\E|Y_0^{(m)}|^q\to\infty,
\qquad
W_p\bigl(\Law(Y_0^{(m)}),p_\delta\bigr)\to\infty.
\]
\end{enumerate}
\end{theorem}

\begin{proof}
Let $R_m\uparrow\infty$ and define the ideal remote-tail denoiser
\[
F_m(x)=-\lambda\Gamma_{R_m}(x)x^3.
\]
Choose a fixed $A>0$ large enough for the cubic amplification estimate below.  Then choose $K_m\uparrow\infty$ so that the terminal shell
\[
[A\sqrt{K_m},A\sqrt{K_m}+1]
\]
lies in the fully active region $|x|\ge R_m+1$.

We first define the approximation tube without referring to the neural parameters.  On the geometric grid, $\rho_j=h_j/t_{j+1}\asymp K_m^{-1}$.  Starting from
$q_{m,0}=A\sqrt{K_m}+1$, define a deterministic upper recursion
\[
q_{m,n+1}=C\bigl(1+q_{m,n}^3\bigr),
\qquad n=0,\ldots,K_m-1,
\]
where $C$ is chosen large enough to dominate one denoiser-form Euler step whenever $|D(x)-F_m(x)|\le1$ and $|\xi_j|\le1$.  Let
\[
Q_m=1+R_m+\max_{0\le n\le K_m}q_{m,n},
\qquad
\zeta_m=Q_m^{-1}.
\]
Then $Q_m<\infty$, $Q_m\to\infty$, and $\zeta_m\le1$.  Lemma~\ref{lem:fixed-dit-embedding} gives parameters $\theta_m$ in one fixed DiT architecture such that
\[
\sup_{|x|\le Q_m}|D_{\theta_m}(x)-F_m(x)|\le\zeta_m.
\]
By induction, on the terminal-shell and bounded-noise event, the neural trajectory stays inside $[-Q_m,Q_m]$.  On the same event, the one-step lower estimate is
\[
|Y_j|
\ge
\rho_j\lambda |Y_{j+1}|^3
-(1-\rho_j)|Y_{j+1}|
-\rho_j\zeta_m
-\sqrt{h_j}.
\]
For the fixed choice of $A$ and all large $m$, this implies
\[
|Y_j|\ge \frac{c}{K_m}|Y_{j+1}|^3
\]
with $c>0$ independent of $m$.  Choosing $A$ so that $cA^2\ge1$ also keeps every iterate in the fully active region.  The rare-event argument of Step~5 in Theorem~\ref{thm:strong-noncert} therefore gives
\[
\E|Y_0^{(m)}|^q\to\infty,
\qquad q>0.
\]

Because the exact denoiser is zero,
\[
\widehat s_t^{(m)}(x)-s_t(x)=\frac{D_{\theta_m}(x)}{t}.
\]
On $[-Q_m,Q_m]$, use the approximation to $F_m$; outside the tube, use \eqref{eq:fixed-dit-poly-bound}.  Since $\zeta_m^{-1}=Q_m$, the global output bound grows at most polynomially in $Q_m$, while the Gaussian tail beyond $Q_m$ decays exponentially in $Q_m^2$.  Since $R_m\to\infty$, the ideal cubic contribution on $\{|x|>R_m\}$ also vanishes under every $p_t$, uniformly for $t\in[\delta,T]$.  Hence
\[
\sup_{t\in[\delta,T]}
\E_{p_t}|\widehat s_t^{(m)}-s_t|^2\to0.
\]
Girsanov's theorem, evaluated under the exact reverse-time path law, gives
\[
\KL\!\left(\Law(U_\cdot)\,\middle\|\,\Law(Z^{(m)}_\cdot)\right)
\le
\frac12\int_\delta^T
\E_{p_t}|\widehat s_t^{(m)}-s_t|^2\,dt,
\]
so Pinsker's inequality yields path-space total variation convergence.  Each fixed learned drift is globally Lipschitz with linear growth on $[\delta,T]$, hence its SDE is nonexplosive and has moments of all orders.

For weak convergence, couple the neural Euler--Maruyama chain with the exact-denoiser Euler--Maruyama chain using the same terminal draw and Gaussian increments.  Since the exact denoiser is zero,
$\bar Y_j/t_j=\bar Y_{j+1}/t_{j+1}+\sqrt{h_j}\xi_j/t_j$.  After reversing the index, this is a Gaussian martingale.  On the geometric grids, $\sum_j h_j/t_j^2$ is uniformly bounded, so Doob's inequality gives grid-uniform maximal moments for $\max_j|\bar Y_j|$.  The exact chain therefore stays in $[-R_m+1,R_m-1]$ with probability tending to one.

On this event, a backward induction gives the comparison with the neural chain.  At the terminal level the two chains agree.  If $|\Delta_{j+1}|\le\zeta_m\le1$, then $|\bar Y_{j+1}|\le R_m-1$ implies $|Y_{j+1}|\le R_m$, so $F_m(Y_{j+1})=0$ and the tube approximation applies.  Writing $\rho_j=h_j/t_{j+1}$, the difference recursion then obeys
\[
|\Delta_j|
\le
(1-\rho_j)|\Delta_{j+1}|+\rho_j\zeta_m.
\]
Thus $\max_j|\Delta_j|\le\zeta_m$ by induction.  Since the exact Euler--Maruyama endpoint converges in probability to the exact reverse-time endpoint, the neural endpoint does as well.  Weak convergence follows, and the Wasserstein divergence follows from moment divergence by the same Minkowski lower bound used in Theorem~\ref{thm:strong-noncert}.
\end{proof}

For each fixed \(m\), boundedness of \(D_{\theta_m}\) gives a moment bound that is uniform in the grid size \(K\), but the bound may depend on \(m\).  The divergence therefore occurs only along the diagonal sequence \((m,K_m)\).

The one-token construction is not essential.  The scalar circuit can be embedded tokenwise in a finite patchified adaLN-Zero DiT, and the resulting fixed circuit fits within the DiT-S/2 resource budget; see Appendices~\ref{app:fixed-neural-realization} and~\ref{app:dit-s2-construction}.

\begin{corollary}[A fixed DiT-S/2 configuration is sufficient]
\label{cor:dit-s2-noncert}
Consider the DiT-S/2 configuration of Peebles and Xie~\cite{peebles2023dit}: $12$ transformer blocks, hidden dimension $384$, $6$ attention heads, MLP ratio $4$, and patch size $2$.  For the standard $32\times32$ latent input with $4$ channels, there exists a sequence of parameter values in this single fixed configuration for which all conclusions of Theorem~\ref{thm:fixed-dit-noncert} hold.
\end{corollary}

\begin{proof}
Appendix~\ref{app:fixed-neural-realization} lifts the scalar construction to the patchified finite-dimensional architecture, while Appendix~\ref{app:dit-s2-construction} shows that the resulting circuit fits within the DiT-S/2 resource budget.  The conclusion then follows from Theorem~\ref{thm:fixed-dit-noncert}.
\end{proof}

\section{Stability from denoiser projection}

The counterexample shows that a small forward-marginal score error need not control the states generated by an explicit sampler.  We now consider a setting with additional geometric information: the data are supported in a known bounded closed convex set \(C\).  Since the exact denoiser
\[
D_t(x)=\E[X_0\mid X_t=x]
\]
also takes values in \(C\), this information can be imposed directly on the learned denoiser without increasing its pointwise approximation error.

Let \(\Pi_C:\mathbb R^d\to C\) denote Euclidean projection and define
\begin{equation}
D_t^P(x)=\Pi_C\widehat D_t(x),
\qquad
s_t^P(x)=\frac{D_t^P(x)-x}{t}.
\label{eq:projected-score}
\end{equation}
The projection preserves the available denoiser accuracy while forcing the deterministic part of each Euler--Maruyama update toward the bounded set \(C\).  This simple modification is enough to recover grid-uniform moment control.

\subsection{Projection removes the amplification mechanism}

Substituting \eqref{eq:projected-score} into \eqref{eq:reverse-em} gives
\begin{equation}
Y_k
=
(1-\rho_k)Y_{k+1}
+
\rho_kD_{t_{k+1}}^P(Y_{k+1})
+
\sqrt{h_k}\,\xi_k.
\label{eq:projected-affine-step}
\end{equation}
For every decreasing positive grid, $0<\rho_k<1$.  The deterministic part is therefore a convex combination of the current state and a point in $C$.  If $C\subset B(0,r_C)$, then for every $y\in\R^d$,
\begin{equation}
\left\|
(1-\rho_k)y
+
\rho_kD_{t_{k+1}}^P(y)
\right\|
\le
(1-\rho_k)\|y\|+\rho_k r_C.
\label{eq:one-step-envelope-control}
\end{equation}
Under \eqref{eq:projected-affine-step}, the deterministic update has at most linear growth toward a bounded set.  A large state can therefore no longer trigger the superlinear amplification cascade from Section~3.

\subsection{Explicit representation and grid-uniform moment control}

The affine update has an explicit expansion.  Since \(1-\rho_k=t_k/t_{k+1}\), products of successive coefficients telescope into ratios of time levels.  This gives the following representation.

\begin{proposition}[Explicit representation and grid-uniform moment bounds]
\label{prop:moment}
Assume $C\subset B(0,r_C)$ and consider a family of grids satisfying
\[
\sup_K\max_{0\le k<K}\rho_k
\le \rho_{\max}<1.
\]
Then each projected chain in the family satisfies
the representation
\begin{align}
Y_k
&=
\frac{t_k}{T}Y_K
+
\sum_{j=k}^{K-1}
w_{k,j}\,
D_{t_{j+1}}^P(Y_{j+1})
+
G_k,
\label{eq:exact-projected-unrolling}\\
w_{k,j}
&=
t_k\left(\frac1{t_j}-\frac1{t_{j+1}}\right),
\qquad
\sum_{j=k}^{K-1}w_{k,j}=1-\frac{t_k}{T},
\label{eq:envelope-convex-weights}\\
G_k
&=
\sum_{j=k}^{K-1}\frac{t_k}{t_j}\sqrt{h_j}\,\xi_j
\sim
N(0,\sigma_k^2I_d),
\qquad
\sigma_k^2
=
t_k^2\sum_{j=k}^{K-1}\frac{h_j}{t_j^2}
\le
\frac{t_k(1-t_k/T)}{1-\rho_{\max}}.
\label{eq:gaussian-variance-bound}
\end{align}
Consequently, for every $q\ge1$,
\begin{equation}
\|Y_k\|_{L^q}
\le
\frac{t_k}{T}\|Y_K\|_{L^q}
+
\left(1-\frac{t_k}{T}\right)r_C
+
\sqrt{\frac{t_k(1-t_k/T)}{1-\rho_{\max}}}\,\|g_d\|_{L^q},
\label{eq:sharp-grid-uniform-moment-bound}
\end{equation}
where $g_d\sim N(0,I_d)$.  The bound is independent of the off-path growth and
regularity of the raw learned score $\widehat s$.
\end{proposition}

\begin{proof}
From \eqref{eq:projected-affine-step},
\[
Y_k
=
\frac{t_k}{t_{k+1}}Y_{k+1}
+
\frac{h_k}{t_{k+1}}
D_{t_{k+1}}^P(Y_{k+1})
+
\sqrt{h_k}\,\xi_k.
\]
Iterating this identity and using telescoping products gives
\eqref{eq:exact-projected-unrolling}.  The denoiser coefficient at level $j$ is
\[
\frac{t_k}{t_j}\frac{h_j}{t_{j+1}}
=
t_k\left(\frac1{t_j}-\frac1{t_{j+1}}\right),
\]
which is nonnegative and telescopes to
\eqref{eq:envelope-convex-weights}.  The Gaussian term has the variance in
\eqref{eq:gaussian-variance-bound}.  Moreover,
\begin{align*}
\sum_{j=k}^{K-1}\frac{h_j}{t_j^2}
&=
\sum_{j=k}^{K-1}
\frac{t_{j+1}}{t_j}
\left(\frac1{t_j}-\frac1{t_{j+1}}\right)\\
&\le
\frac1{1-\rho_{\max}}
\left(\frac1{t_k}-\frac1T\right),
\end{align*}
which proves the stated variance bound.  Finally,
$D_t^P(\cdot)\in C$, and the weights in
\eqref{eq:envelope-convex-weights} are nonnegative.  Hence the middle term in
\eqref{eq:exact-projected-unrolling} belongs to
$(1-t_k/T)C$ and has norm at most $(1-t_k/T)r_C$.  Taking $L^q$ norms and
applying Minkowski's inequality proves
\eqref{eq:sharp-grid-uniform-moment-bound}.
\end{proof}

For fixed $p\ge1$, take $q>p$ in Proposition~\ref{prop:moment}.  Then $\|Y_k\|^p$ is uniformly integrable over the grid family.  This is the property that fails in the counterexample of Section~3.

\subsection{Stability without loss of score accuracy}

Projection also preserves approximation accuracy.  The exact denoiser already takes values in \(C\), so Euclidean projection cannot move the learned denoiser farther from the target.  The stability improvement established below therefore comes without an increase in approximation error.

\begin{proposition}[Projection contracts denoiser and score error]
\label{prop:projection-contraction}
Let $C\subset\R^d$ be closed and convex, with $\operatorname{supp}(\mu_0)\subset C$.  Then $D_t(x)\in C$ for every $t>0$ and $x\in\R^d$.  Consequently,
\[
\|D_t^P(x)-D_t(x)\|
\le
\|\widehat D_t(x)-D_t(x)\|,
\]
and equivalently,
\[
\|s_t^P(x)-s_t(x)\|
\le
\|\widehat s_t(x)-s_t(x)\|.
\]
Hence
\[
\E_{p_t}\|s_t^P-s_t\|^2
\le
\E_{p_t}\|\widehat s_t-s_t\|^2.
\]
\end{proposition}

\begin{proof}
The posterior law of $X_0$ given $X_t=x$ is supported in $C$, so its barycenter $D_t(x)$ also lies in $C$.  Euclidean projection onto a closed convex set decreases distance to every point in that set, so
\[
\|\Pi_C\widehat D_t(x)-D_t(x)\|
\le
\|\widehat D_t(x)-D_t(x)\|.
\]
The score inequality follows by dividing by $t$ and using Tweedie's identity.
\end{proof}

\begin{theorem}[Projection restores Wasserstein convergence]
\label{cor:projected-wp-recovery}
Let $C\subset B(0,r_C)$ be bounded, closed, and convex, with $\operatorname{supp}(\mu_0)\subset C$.  Consider a sequence of grids with mesh size tending to zero and $\rho_k\le\rho_{\max}<1$, and initialize each projected Euler--Maruyama chain with $Y_K^{(K)}\sim p_T$.  Assume that $D^P_t$ is jointly continuous and locally Lipschitz in $x$, locally uniformly over $t\in[\delta,T]$, and let $Z^P$ solve
\[
dZ^P_u=s^P_{T-u}(Z^P_u)\,du+dB_u,
\qquad Z^P_0\sim p_T.
\]
Then, for every $p\ge1$,
\[
W_p\!\left(\Law(Y_0^{(K)}),\Law(Z^P_{T-\delta})\right)\longrightarrow0.
\]
Moreover, define the projected on-path error
\[
\mathcal E_P
=
\int_\delta^T
\E_{p_t}\|s_t^P-s_t\|^2\,dt.
\]
For every $q>p$, there is a finite constant $C_{p,q}$, depending also on the fixed endpoint $q$th moments, such that
\begin{equation}
W_p\!\left(\Law(Y_0^{(K)}),p_\delta\right)
\le
o_K(1)
+
C_{p,q}\,\mathcal E_P^{\frac12(\frac1p-\frac1q)}.
\label{eq:projected-wp-certificate}
\end{equation}
Here $o_K(1)\to0$ as the mesh size tends to zero.  By Proposition~\ref{prop:projection-contraction}, $\mathcal E_P$ is no larger than the raw on-path score error.
\end{theorem}

\begin{proof}
Since $D_t^P$ is bounded and $t\ge\delta$, the projected score has linear growth.  The stated local regularity gives localized Euler convergence under the usual Brownian-increment coupling,
\[
Y_0^{(K)}\longrightarrow Z^P_{T-\delta}
\qquad\text{in probability}.
\]
Apply Proposition~\ref{prop:moment} with any $q>p$.  The resulting grid-uniform $q$th-moment bound makes $\|Y_0^{(K)}\|^p$ uniformly integrable.  The limit $Z^P_{T-\delta}$ also has a finite $q$th moment.  Hence the coupled endpoint differences converge to zero in $L^p$, which is stronger than convergence in $W_p$.

For the comparison with the exact endpoint, both $D_t^P$ and $D_t$ take values in $C$, so $s_t^P-s_t=(D_t^P-D_t)/t$ is bounded on $[\delta,T]\times\mathbb R^d$.  Thus the Girsanov change of measure is valid, and evaluation under the exact reverse-time law gives
\[
\KL\!\left(
\Law(U_\cdot)\,\middle\|\,\Law(Z^P_\cdot)
\right)
\le
\frac12\mathcal E_P.
\]
Pinsker's inequality yields
\[
\TV\!\left(\Law(U_{T-\delta}),\Law(Z^P_{T-\delta})
\right)
\le \frac12\sqrt{\mathcal E_P}.
\]
Both endpoint laws have finite moments of every fixed order.  The standard TV--moment interpolation inequality then gives, for $q>p$,
\[
W_p\!\left(\Law(Z^P_{T-\delta}),p_\delta\right)
\le
C_{p,q}\,\mathcal E_P^{\frac12(\frac1p-\frac1q)}.
\]
The triangle inequality and the first part of the corollary prove \eqref{eq:projected-wp-certificate}.
\end{proof}

\section{Experiments}
\label{sec:experiments}
The experiments provide a constructive realization test.  We ask whether one fixed finite neural architecture can remain accurate under ordinary sampling, reproduce the amplification mechanism of the counterexample, and suppress this amplification after denoiser projection. We use an $8\times8$ one-channel point-mass variance-exploding model on $t\in[0.1,1]$, so the exact denoiser is zero.  Four family members share the same small DiT-style architecture and differ only in their parameters and remote target profiles.  Training and evaluation details are given in Appendix~\ref{app:experimental-protocol}.

\subsection{Unconditioned accuracy and post-entry amplification}

We first evaluate the selected checkpoints on unconditioned trajectories.  Table~\ref{tab:neural-summary} reports forward-marginal score MSE, paired trajectory RMSE, and the stress-test RMS amplification defined below.  The score MSE is at most $7.10\times10^{-3}$, while the trajectory RMSE ranges from $0.012$ to $0.024$.

\begin{table}[!ht]
\centering
\small
\begin{tabular}{cccc}
\toprule
Family & score MSE & path RMSE & RMS amplification $r_8/r_0$ \\
\midrule
$m_0$ & $1.94\!\times\!10^{-3}$ & $0.0119$ & $7.90$ \\
$m_1$ & $3.73\!\times\!10^{-3}$ & $0.0167$ & $9.78$ \\
$m_2$ & $3.89\!\times\!10^{-3}$ & $0.0171$ & $13.02$ \\
$m_3$ & $7.10\!\times\!10^{-3}$ & $0.0237$ & $14.69$ \\
\bottomrule
\end{tabular}
\caption{Unconditioned-trajectory accuracy and neural amplification for the selected checkpoints.}
\label{tab:neural-summary}
\end{table}

We use a stress-initialized protocol to isolate the dynamics after the numerical chain reaches the stress region.  This protocol does not estimate the natural entry probability.  Let $r_k=(\widehat{\E}|Y_k|^2)^{1/2}$ be the empirical coordinatewise root-mean-square (RMS) magnitude after step $k$.  Across the family, $r_8/r_0$ rises from $7.9$ to $14.7$, while the neural endpoint retains about $90\%$--$95\%$ of the target RMS magnitude.  Figure~\ref{fig:neural-stress-cascade} shows the stepwise growth and close tracking of the target recursion.

The trajectories also show a simple feedback mechanism: once the chain reaches a region where forward-marginal error gives little control, one inaccurate update can move later score queries farther from the reference path and amplify the deviation.  Here this produces repeated the amplification.

\begin{figure}[!t]
\centering
\includegraphics[width=0.86\linewidth]{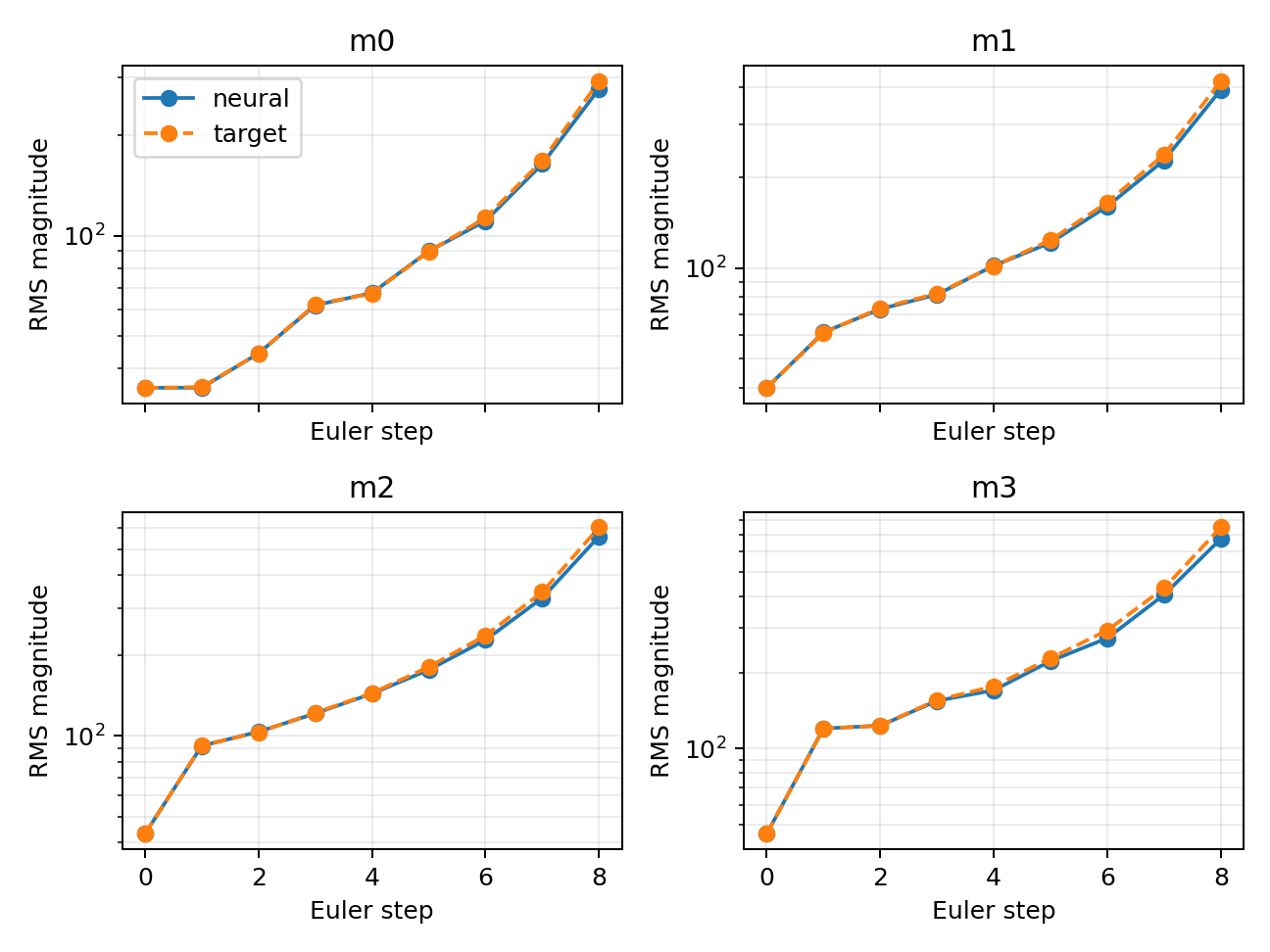}
\caption{Stepwise coordinatewise RMS magnitude $r_k=(\widehat{\E}|Y_k|^2)^{1/2}$ for the neural and paired target-profile chains.}
\label{fig:neural-stress-cascade}
\end{figure}

\subsection{Projection suppresses post-entry amplification}

We next apply denoiser projection with $C=[-1,1]$.  Raw and projected chains use identical stress initializations and bounded perturbations.  Projection reduces the final empirical second moment by factors between $1.13\times10^4$ and $3.33\times10^4$.  Figure~\ref{fig:neural-projection-stress} shows the stepwise separation: raw RMS magnitude grows, whereas projected RMS magnitude falls from $35$--$46$ initially to $2.6$--$3.7$.  Thus projection changes the post-entry dynamics and removes the amplification mechanism in this stress test.

\begin{figure}[H]
\centering
\includegraphics[width=0.86\linewidth]{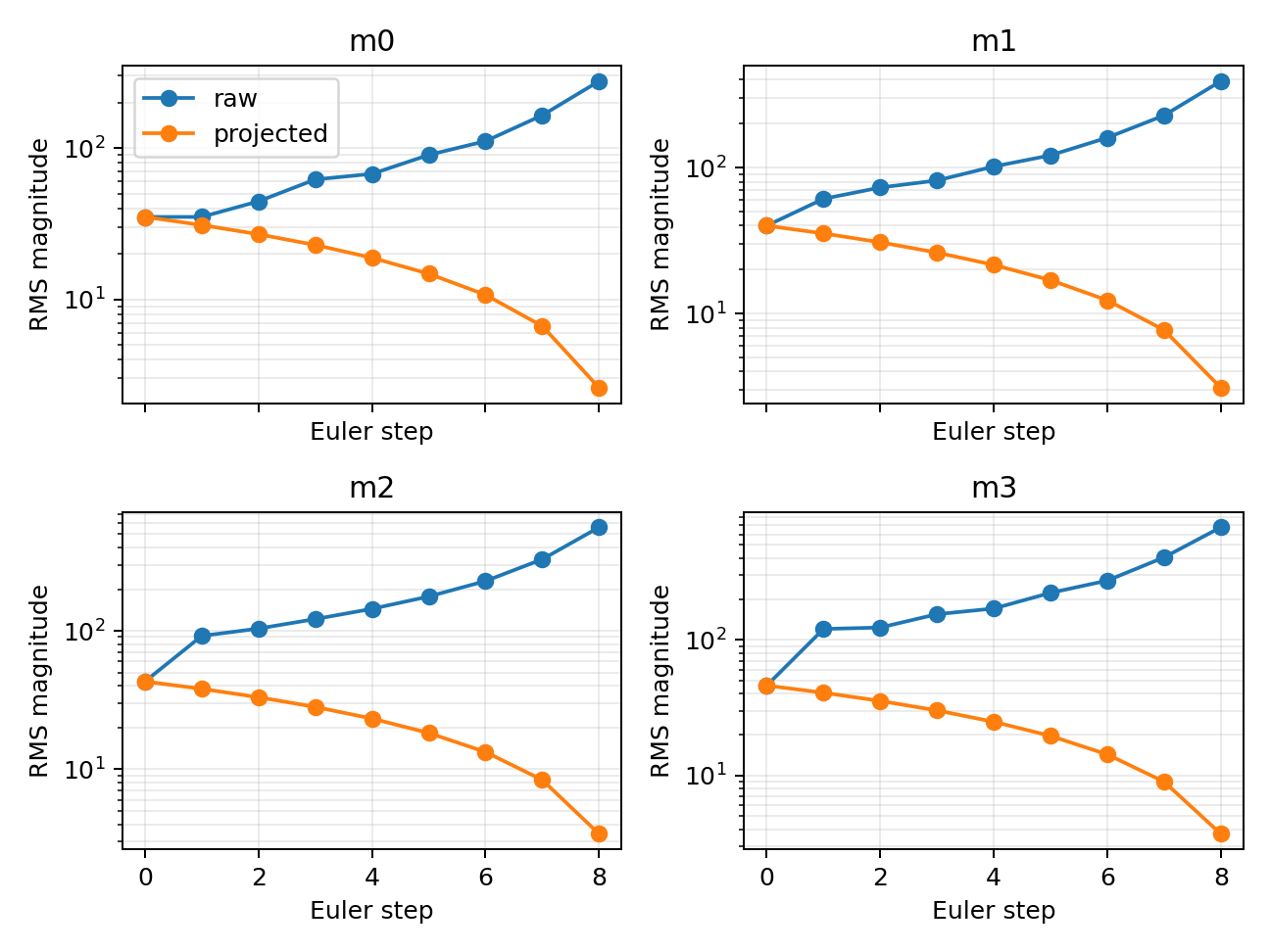}
\caption{Stepwise coordinatewise RMS magnitude for the paired raw and projected chains under the stress-initialized protocol.}
\label{fig:neural-projection-stress}
\end{figure}

\section{Discussion}

The results point to a gap between forward-marginal score accuracy and stability along numerical trajectories.  The problem is not only that the learned score may be inaccurate at states rarely visited under the forward marginals.  A numerical error also changes where the score is evaluated at later steps.  If the new state lies in a poorly controlled region, the next update can increase the deviation further.  Our counterexample turns this simple feedback mechanism into a superlinear cascade, but the same concern can arise more generally in learned reverse-time dynamics.  A useful certificate should therefore provide control that is uniform across the relevant learned family, control the score at states visited by the numerical sampler, or constrain the numerical transition directly.  Projection gives one such structural safeguard when a bounded convex set containing the data support is known.

A natural next step is to extend this framework to Schr\"odinger bridges and related learned transport dynamics.  In Schr\"odinger bridge problems, time-dependent drifts or controls are learned along bridge marginals and then simulated numerically, so a similar gap between bridge-marginal accuracy and solver-generated trajectories may arise.  Two questions are direct: whether small bridge-marginal error can coexist with stable continuous dynamics but unstable discretization, and whether bridge geometry or feasible-state constraints can yield projection-like stabilization.

\appendix

\section{Fixed-size neural realization and patchified embedding}
\label{app:fixed-neural-realization}

This appendix gives the network details used in Section~3.  The first two lemmas build the fixed-size remote-tail circuit and embed it into DiT; the final part handles positional embeddings for the patchified finite-dimensional lifting used in Corollary~\ref{cor:dit-s2-noncert}.

\begin{lemma}[Fixed-size realization of the remote cubic profile]
\label{lem:fixed-gelu-profile}
Fix $\lambda>0$.  There exist integers $L,W<\infty$, independent of $R,Q$ and the approximation tolerance, and a scalar network architecture $\mathcal A_{L,W}$ built from affine maps, GELU activations, residual connections, and a final fixed-dimensional LayerNorm--linear head, with the following property.  For every $1<R<Q$ and every $\zeta>0$, one can choose parameters $\theta$ in this same architecture so that the resulting denoiser $D_\theta$ is globally bounded and globally Lipschitz and
\begin{equation}
\sup_{|x|\le Q}
\left|D_\theta(x)+\lambda \Gamma_R(x)x^3\right|
\le \zeta.
\label{eq:fixed-gelu-profile}
\end{equation}
The depth $L$ and width $W$ do not depend on $R,Q$, or $\zeta$; only the parameter values do.
\end{lemma}

\begin{proof}
We give an explicit construction.  Let $\sigma$ be the tanh-approximated GELU activation used in the DiT MLP\@.  We need only two facts: $\sigma(z)-z_+$ is bounded on $\R$, and $\sigma''$ is nonzero somewhere.  The first follows from the tanh-GELU formula and its limits at $\pm\infty$.  The second follows because the activation is smooth and non-affine.

First,
\[
\mathcal R_k(u)=k^{-1}\sigma(ku)
\]
approximates $u_+$ uniformly on $\R$ as $k\to\infty$.  Indeed, with
\[
C_\sigma=\sup_{z\in\R}|\sigma(z)-z_+|<\infty,
\]
one has
\[
\sup_{u\in\R}|\mathcal R_k(u)-u_+|
\le \frac{C_\sigma}{k}.
\]
Hence a fixed four-unit GELU layer approximates the two-sided gate $\Gamma_R$ uniformly, with the location $R$ encoded only through biases.  Here $\Gamma_R$ is the two-sided ramp gate defined in Section~3.

Second, a fixed-size GELU module can approximate multiplication on any compact box.  Its weights may depend on the box and the tolerance.  Choose $a_\sigma$ with $\sigma''(a_\sigma)\neq0$ and set
\[
S_h(u)
=
\frac{\sigma(a_\sigma+hu)+\sigma(a_\sigma-hu)-2\sigma(a_\sigma)}{h^2\sigma''(a_\sigma)}.
\]
Taylor's theorem gives, for every $H<\infty$,
\[
\sup_{|u|\le H}|S_h(u)-u^2|
\le C_H h^2.
\]
Therefore
\[
\operatorname{Mult}_h(u,v)
=
\frac14\bigl(S_h(u+v)-S_h(u-v)\bigr)
\]
approximates $uv$ uniformly on $[-H,H]^2$.  This module uses a fixed number of GELU units; only $h$ and the affine weights vary with $H$ and the target accuracy.

A fixed finite composition of these modules computes, to arbitrary accuracy on $[-Q,Q]$, $x^2$, $x^3$, the gate $\Gamma_R(x)$, and finally $\Gamma_R(x)x^3$.  The number of multiplication modules is fixed.  Hence the depth and width do not depend on $R,Q$, or $\zeta$.

It remains to make the output globally bounded without changing the architecture size.  Let $\varepsilon_{\mathrm{LN}}>0$ be the fixed LayerNorm stabilizer and set
\[
q(z)=(z,1,-1),
\qquad
 H_{\mathrm{LN}}(z)=\bigl[\mathrm{LN}_{\varepsilon_{\mathrm{LN}}}(q(z))\bigr]_1.
\]
A direct calculation shows that $H_{\mathrm{LN}}$ is smooth, bounded, and strictly increasing.  Near zero, $H_{\mathrm{LN}}(z)=\kappa z+O(z^3)$ for some $\kappa>0$.  The target range is bounded by $|\lambda\Gamma_R(x)x^3|\le \lambda Q^3$.  Choose a large scale $c$ and feed the arithmetic circuit output, divided by $c\kappa$, into the fixed LayerNorm head $c H_{\mathrm{LN}}(\cdot)$.  Then
\[
c\,H_{\mathrm{LN}}\!\left(\frac{y}{c\kappa}\right)=y+O\!\left(\frac{|y|^3}{c^2}\right)
\]
uniformly for $|y|\le \lambda Q^3$.  Taking $c$ large and the internal arithmetic tolerances small yields \eqref{eq:fixed-gelu-profile}.  The final LayerNorm head is bounded, and every component of the finite network is globally Lipschitz for fixed parameters.  Hence $D_\theta$ is globally bounded and globally Lipschitz.
\end{proof}

\begin{lemma}[Embedding the fixed circuit into a fixed adaptive LayerNorm-Zero (adaLN-Zero) DiT]
\label{lem:fixed-dit-embedding}
Fix $\lambda>0$.  There exist integers $N,d,r<\infty$, independent of $R,Q$, and the approximation tolerance, and a scalar one-token adaLN-Zero DiT architecture with $N$ blocks, token dimension $d$, and MLP expansion width $r$, such that the following holds.  For every $1<R<Q$ and every $\zeta>0$, one can choose parameters in this same architecture so that its scalar denoiser output $D_\theta$ is globally bounded and globally Lipschitz and
\begin{equation}
\sup_{|x|\le Q}
\left|D_\theta(x)+\lambda\Gamma_R(x)x^3\right|
\le \zeta.
\label{eq:fixed-dit-profile}
\end{equation}
Only the parameter values vary with $R,Q$, and $\zeta$.  There are also constants $C,a,b<\infty$, depending only on the fixed architecture and $\lambda$, such that the parameters can be chosen with
\begin{equation}
\|D_\theta\|_\infty
\le C(1+Q)^a(1+\zeta^{-1})^b.
\label{eq:fixed-dit-poly-bound}
\end{equation}
\end{lemma}

\begin{proof}
We now show how a fixed stack of adaLN-Zero MLP blocks simulates the GELU arithmetic circuit from Lemma~\ref{lem:fixed-gelu-profile}.  In a standard adaLN-Zero block, the conditioning network produces shift, scale, and gate vectors for the attention and MLP branches.  Set the attention gates to zero.  Choose the affine modulation layers so that the MLP shift and scale are zero and the MLP gate is a fixed nonzero constant.  This is done by zeroing the relevant conditioning weights and using only their biases.  The timestep and class embeddings are then ignored.  After absorbing the constant gate into the output weights, each active block has the form
\[
h\longmapsto h+M_\ell(\mathrm{LN}_{\varepsilon_{\mathrm{LN}}}(h)),
\]
where $M_\ell$ is a two-layer affine--GELU--affine map and $\varepsilon_{\mathrm{LN}}>0$ is the fixed LayerNorm stabilizer.  This is exactly a parameter subfamily of the adaLN-Zero block used in DiT~\cite{peebles2023dit}.

The remaining issue is pre-normalization.  Let $d_{\mathrm{ws}}$ be the workspace dimension of the arithmetic circuit and choose $d\ge d_{\mathrm{ws}}+2$.  Fix an anchor $a\in\R^d$ with nonzero centered part.  Choose an isometric embedding $U:\R^{d_{\mathrm{ws}}}\to\R^d$ whose range is orthogonal to both the constant vector $\mathbf 1$ and the centered anchor $a-\bar a\mathbf 1$.  For a bounded workspace state $z$ and a small signal scale $\epsilon>0$, encode
\[
h=a+\epsilon Uz.
\]
A Taylor expansion of LayerNorm at $a$ gives, uniformly for $z$ in a fixed bounded box,
\begin{equation}
\mathrm{LN}_{\varepsilon_{\mathrm{LN}}}(a+\epsilon Uz)
=
\ell_0+c_a\epsilon Uz+O(\epsilon^2),
\qquad c_a>0,
\label{eq:ln-small-signal}
\end{equation}
where the orthogonality choice makes the derivative a scalar multiple of the identity on the signal subspace.

Fix $R,Q$, and a target accuracy.  Rescale the finitely many workspace coordinates so that every intermediate state on $[-Q,Q]$ lies in one bounded box.  These diagonal rescalings are absorbed into neighboring affine weights and do not change the circuit size.  Now consider one arithmetic stage
\[
z\longmapsto z+\Delta_\ell(z),
\]
where $\Delta_\ell$ is produced by one affine--GELU--affine module of the fixed circuit.  By \eqref{eq:ln-small-signal}, the first affine map of a DiT MLP can subtract $\ell_0$ and multiply by $(c_a\epsilon)^{-1}U^\top$.  This recovers $z$ up to $O(\epsilon)$ error on the relevant compact set.  Use the arithmetic-stage weights in the middle of the block, then multiply the final affine output by $\epsilon U$.  The result is
\[
a+\epsilon Uz
\longmapsto
 a+\epsilon U\bigl(z+\Delta_\ell(z)\bigr)+o(\epsilon)
\]
uniformly on that set.  Since the circuit has finitely many stages, choosing $\epsilon$ small enough makes the full fixed stack approximate its encoded workspace trajectory to any target accuracy.

The affine patch embedding initializes the anchor and the rescaled input coordinate.  Store the final circuit output in one bounded workspace coordinate after dividing by a parameter-dependent scale $S_{\mathrm{out}}\ge1+\lambda Q^3$.  Apply \eqref{eq:ln-small-signal} once more.  The final LayerNorm and a linear readout with weight of order $S_{\mathrm{out}}/\epsilon$ recover the unscaled scalar output with arbitrarily small uniform error on $[-Q,Q]$.  The number of blocks, token dimension, and MLP width are fixed by the arithmetic circuit.  They do not depend on $R,Q$, or $\zeta$.

For every fixed parameter choice, affine maps and GELU are globally Lipschitz, LayerNorm with positive stabilizing constant is globally Lipschitz, and the final LayerNorm has bounded coordinates.  Hence the scalar DiT output is globally bounded and globally Lipschitz.

Each arithmetic module can be chosen with parameters polynomial in $1+Q$ and $\zeta^{-1}$.  Since the circuit has fixed depth, composing the modules preserves polynomial dependence for the intermediate weights, workspace rescalings, and Lipschitz constants.  The small-signal LayerNorm simulation and final decoding require only additional polynomial rescaling.  Hence there are constants $C,a,b<\infty$ for which \eqref{eq:fixed-dit-poly-bound} holds.
\end{proof}

\subsection{Patchified embedding and positional control}

We now give the patchified finite-dimensional lifting used in Corollary~\ref{cor:dit-s2-noncert}.

We embed the scalar construction tokenwise.  Write the latent input as a fixed collection of $N_{\mathrm{tok}}<\infty$ patches and choose one scalar coordinate $u_i$ in each patch.  Set the attention gates to zero.  The remaining MLP branch uses shared weights across tokens.  The final linear head is also tokenwise, and unpatchification only rearranges patch outputs.

The only new issue is the fixed positional embedding.  Let $p_i\in\R^d$ be the positional vector of token $i$.  Let $e_x\in\R^{d_{\mathrm{ws}}}$ be the workspace basis vector for the input coordinate and set $v:=Ue_x\in\R^d$.  Choose positive scales $A_m$ and $\epsilon_m$.  The shared affine patch embedding is chosen so that, on the selected scalar coordinate,
\[
E_m(u_i)=A_m\bigl(a+\epsilon_m v u_i\bigr),
\]
with the same anchor $a$ and signal direction $U$ used in the small-signal construction.  The input to the first block is then
\[
A_m\bigl(a+\epsilon_m v u_i\bigr)+p_i.
\]
Factoring out the anchor scale gives the exact identity
\[
\mathrm{LN}_{\varepsilon_{\mathrm{LN}}}\!\left(
A_m(a+\epsilon_m v u_i)+p_i
\right)
=
\mathrm{LN}_{\varepsilon_{\mathrm{LN}}/A_m^2}\!\left(
 a+\epsilon_m v u_i+A_m^{-1}p_i
\right).
\]
The centered anchor has nonzero variance, so zero-stabilizer normalization is smooth near the compact encoded tube.  Since there are finitely many tokens,
\[
\mathrm{LN}_{\varepsilon_{\mathrm{LN}}/A_m^2}\!\left(
 a+\epsilon_m v u_i+A_m^{-1}p_i
\right)
=
\mathrm{LN}_{0}\!\left(a+\epsilon_m v u_i\right)
+O(A_m^{-1})
\]
uniformly for $|u_i|\le Q$ and over all tokens.  The same orthogonality choice as in Lemma~\ref{lem:fixed-dit-embedding} makes the derivative of $\mathrm{LN}_0$ a scalar multiple of the identity on the signal subspace.

The first affine map in an active MLP block decodes the workspace with gain of order $\epsilon_m^{-1}$, so the positional perturbation creates decoded error of order $(A_m\epsilon_m)^{-1}$.  Choose compatible scales with $\epsilon_m\to0$ and
\[
A_m\epsilon_m\longrightarrow\infty.
\]
Scaling each active residual output by $A_m\epsilon_m$ then gives
\[
A_m(a+\epsilon_m Uz)
\longmapsto
A_m\bigl(a+\epsilon_m U(z+\Delta_\ell(z))\bigr)
+o(A_m\epsilon_m),
\]
uniformly on the compact tube.  Thus the fixed positional vectors remain a controlled perturbation of the tokenwise circuit.

Choose the final tokenwise linear head so that only the selected scalar slot of each reconstructed patch is nonzero.  The resulting ideal denoiser is the product map
\[
F_m^{\otimes}(x)
=
\bigl(F_m(u_1),\ldots,F_m(u_{N_{\mathrm{tok}}})\bigr)
\]
on the selected patch coordinates, with all remaining output coordinates zero, up to an arbitrarily small compact-tube error.  The number of tokens and the patch dimension are fixed.  Thus the on-path $L^2$ budget and the integrated drift-error term in Girsanov's bound change only by a fixed dimensional factor.  Taking $R_m\to\infty$ still gives vanishing score error and vanishing path-space total variation distance.

For the numerical lower bound, place one selected patch coordinate in the moving terminal shell and impose the bounded-noise event only on that coordinate.  Attention is disabled and the tokenwise map is coordinate-separable.  The selected coordinate therefore follows the same perturbed scalar recursion as in Theorem~\ref{thm:fixed-dit-noncert}, while the other coordinates do not affect its update.  Scalar rare-event amplification then forces every positive moment of the full latent norm to diverge.  The weak-convergence argument is unchanged after replacing scalar maximal moments by finite-dimensional ones.

\section{Embedding into DiT-S/2}
\label{app:dit-s2-construction}

This appendix checks that the fixed circuit of Appendix~\ref{app:fixed-neural-realization} fits inside the concrete DiT-S/2 configuration of Peebles and Xie~\cite{peebles2023dit}.

\begin{proposition}[The fixed circuit fits inside DiT-S/2]
\label{prop:dit-s2-constructive}
Fix $\lambda>0$.  Consider the DiT-S/2 backbone with $12$ adaLN-Zero transformer blocks, token dimension $384$, MLP expansion width $1536$, six attention heads, and patch size $2$.  For every $1<R<Q$ and every $\zeta\in(0,1)$, one can choose parameters in this single fixed architecture so that, on a selected scalar coordinate $u_i$ of every patch token $i$,
\begin{equation}
\sup_{|u_i|\le Q}
\left|D_{\theta,i}(u_i)+\lambda\Gamma_R(u_i)u_i^3\right|
\le \zeta,
\label{eq:dit-s2-constructive-approx}
\end{equation}
while all remaining output coordinates are zero.  The attention branches are inactive, exactly three MLP residual blocks are active, and the remaining nine transformer blocks are identities.  The architecture size is independent of $R,Q$, and $\zeta$.
\end{proposition}

\begin{proof}
Appendix~\ref{app:fixed-neural-realization} realizes the target profile with a fixed tokenwise circuit consisting of one gate module, one square module, and two multiplication modules, all simulated through a fixed small-signal LayerNorm workspace.  Apply that construction tokenwise with shared weights.  The first active MLP block computes the gate and square terms in parallel, the second forms the cubic term, and the third multiplies by the gate and applies the factor $-\lambda$.  Closing all attention gates makes the computation coordinate-separable, and setting both residual gates to zero makes the remaining nine blocks exact identities.

The resource check is fixed.  The first active stage needs at most six GELU hidden units and the next two need four each, well below the available MLP width $1536$.  The workspace and anchor construction of Lemma~\ref{lem:fixed-dit-embedding} has fixed dimension strictly below the token width $384$, and the patchified embedding argument of Appendix~\ref{app:fixed-neural-realization} controls the positional perturbation.  Its scale choices make the compact-tube error at most $\zeta$ without changing any architecture dimension, which proves \eqref{eq:dit-s2-constructive-approx}.
\end{proof}

\section{Experimental protocol}
\label{app:experimental-protocol}

The neural runs use one fixed training seed (1729).  Unconditioned evaluations use common evaluation seeds across the family.  In the stress experiments, terminal states and perturbations are paired within each neural--target or raw--projected comparison.  Summaries on unconditioned trajectories use $1024$ trajectories per grid and $256$ paired paths for the fine-grid coupling diagnostic.  Stress and projection experiments use $512$ paired trajectories.  The reported values are Monte Carlo point estimates conditional on this training run and the selected checkpoints.  Path RMSE is evaluated on a $512$-step grid.  The stress-initialized protocol uses eight Euler-form updates with standardized perturbations sampled uniformly from $[-0.35,0.35]$.  It is designed to isolate post-entry dynamics and does not estimate the natural probability of reaching the stress region.

The base model is an $8\times8$ one-channel point-mass diffusion model,
\[
X_0=0,
\qquad
X_t=\sqrt t\,g,
\qquad
D_t\equiv0,
\qquad
t\in[0.1,1].
\]
All learned denoisers use the same small DiT-style architecture: patch size $2$, hidden dimension $64$, and three adaptive-LayerNorm/GELU residual MLP blocks, for $104{,}388$ trainable parameters.  Family members differ only in their parameter values and the remote target profile used for training.

Coordinatewise, the target is a smoothly gated cubic with tanh saturation,
\[
D_m^{\mathrm{target}}(x)
=
L_m\tanh\!\left(
-\frac{\lambda\,\gamma_{R_m}(x)x|x|^2}{L_m}
\right),
\qquad \lambda=0.02,
\]
where $\gamma_R(x)=q((|x|-R)/0.5)$, $q(u)=\bar u^2(3-2\bar u)$, and $\bar u=\min\{1,\max\{0,u\}\}$.  Thus $\gamma_R$ rises from $0$ to $1$ over a transition width of $0.5$.  The four configurations are
\[
(R_m,L_m,y_m^{\mathrm{stress}})
=
(2.5,700,35),
(3,1000,40),
(3.5,1450,43),
(4,1800,46),
\]
where $y_m^{\mathrm{stress}}$ is the center of the terminal stress initialization in the post-entry stress experiment.

Training combines four terms: on-path score loss, safe-region denoiser loss, cascade-tube regression, and transition consistency for the deterministic Euler relation.  After joint training, a common calibration rule filters checkpoints by fixed upper bounds on cascade-tube and transition-consistency errors.  Among the surviving checkpoints, it selects the checkpoint with the smallest on-path score error.  All reported main-text comparisons use these selected checkpoints and paired random draws across the family.

For Table~\ref{tab:neural-summary}, score MSE is evaluated after checkpoint selection with fresh randomness (evaluation seed 10729) and is the Monte Carlo mean of $|\widehat s_t(X_t)-s_t(X_t)|^2$ for $t\sim\operatorname{Unif}[0.1,1]$ and $X_t\sim N(0,tI)$.  Path RMSE uses paired fine-grid trajectories: for each pair, we compute the coordinate-averaged squared learned--exact error at every grid point, take the maximum over the path, average over trajectories, and then take the square root.  In the stress experiments, $r_k=(\widehat{\E}|Y_k|^2)^{1/2}$ is the empirical coordinatewise RMS magnitude at Euler step $k$.  The amplification factor is $r_8/r_0$, the neural--target fidelity statistic is $r_8^{\mathrm{neural}}/r_8^{\mathrm{target}}$, and the projection reduction factor is the ratio of the final raw and projected empirical second moments.


\begin{thebibliography}{99}

\bibitem{song2021sde}
Y. Song, J. Sohl-Dickstein, D. P. Kingma, A. Kumar, S. Ermon, and B. Poole.
\newblock Score-Based Generative Modeling through Stochastic Differential Equations.
\newblock \emph{International Conference on Learning Representations}, 2021.

\bibitem{chen2022minimal}
S. Chen, S. Chewi, J. Li, Y. Li, A. Salim, and A. Zhang.
\newblock Sampling is as easy as learning the score: theory for diffusion models with minimal data assumptions.
\newblock \emph{International Conference on Learning Representations}, 2023.

\bibitem{lee2022general}
H. Lee, J. Lu, and Y. Tan.
\newblock Convergence of Score-Based Generative Modeling for General Data Distributions.
\newblock \emph{Proceedings of the 34th International Conference on Algorithmic Learning Theory}, PMLR 201:946--985, 2023.

\bibitem{benton2024linear}
J. Benton, V. De Bortoli, A. Doucet, and G. Deligiannidis.
\newblock Nearly $d$-Linear Convergence Bounds for Diffusion Models via Stochastic Localization.
\newblock \emph{International Conference on Learning Representations}, 2024.

\bibitem{gao2023wasserstein}
X. Gao, H. M. Nguyen, and L. Zhu.
\newblock Wasserstein Convergence Guarantees for a General Class of Score-Based Generative Models.
\newblock \emph{Journal of Machine Learning Research}, 26(43):1--54, 2025.

\bibitem{brunosabanis2025}
S. Bruno and S. Sabanis.
\newblock Wasserstein Convergence of Score-based Generative Models under Semiconvexity and Discontinuous Gradients.
\newblock \emph{Transactions on Machine Learning Research}, 2025.

\bibitem{cao2026robustness}
D. Y. Cao, A. Y. Chen, K. Sridharan, and Y. Wu.
\newblock On the Robustness of Langevin Dynamics to Score Function Error.
\newblock arXiv:2603.11319, 2026.

\bibitem{hjk2011}
M. Hutzenthaler, A. Jentzen, and P. E. Kloeden.
\newblock Strong and weak divergence in finite time of Euler's method for stochastic differential equations with non-globally Lipschitz continuous coefficients.
\newblock \emph{Proceedings of the Royal Society A}, 467:1563--1576, 2011.

\bibitem{khelifa2026gradients}
N. B. Khelifa, R. E. Turner, and R. Venkataramanan.
\newblock Diffusion Models Observe Only Gradients: A Geometric Perspective on Score Matching Errors.
\newblock arXiv:2606.06179, 2026.

\bibitem{tangzhao2024contractive}
W. Tang and H. Zhao.
\newblock Contractive Diffusion Probabilistic Models.
\newblock arXiv:2401.13115, 2024.

\bibitem{conforti2025kl}
G. Conforti, A. Durmus, and M. Gentiloni Silveri.
\newblock KL Convergence Guarantees for Score Diffusion Models under Minimal Data Assumptions.
\newblock \emph{SIAM Journal on Mathematics of Data Science}, 7(1):86--109, 2025.

\bibitem{peebles2023dit}
W. Peebles and S. Xie.
\newblock Scalable Diffusion Models with Transformers.
\newblock \emph{Proceedings of the IEEE/CVF International Conference on Computer Vision}, 2023.

\bibitem{saharia2022imagen}
C. Saharia, W. Chan, S. Saxena, L. Li, J. Whang, E. Denton, K. Ghasemipour, R. Gontijo Lopes, B. Karagol Ayan, T. Salimans, J. Ho, D. J. Fleet, and M. Norouzi.
\newblock Photorealistic Text-to-Image Diffusion Models with Deep Language Understanding.
\newblock arXiv:2205.11487, 2022.

\bibitem{lou2023reflected}
A. Lou and S. Ermon.
\newblock Reflected Diffusion Models.
\newblock \emph{Proceedings of the 40th International Conference on Machine Learning}, PMLR 202:22675--22701, 2023.

\bibitem{highammaostuart2002}
D. J. Higham, X. Mao, and A. M. Stuart.
\newblock Strong convergence of Euler-type methods for nonlinear stochastic differential equations.
\newblock \emph{SIAM Journal on Numerical Analysis}, 40(3):1041--1063, 2002.

\end{thebibliography}
\end{document}